\documentclass{llncs}
\usepackage{times}
\usepackage{helvet}
\usepackage{courier}

\usepackage{amssymb}
\usepackage{amsmath}
\usepackage{algorithm}
\usepackage{algpseudocode}
\usepackage{graphicx}

\usepackage{slashbox}
\usepackage{multirow}
\usepackage{booktabs}
\usepackage[figuresright]{rotating}

\makeatletter
\@addtoreset{lemma}{section}
\@addtoreset{theorem}{section}
\makeatother

\begin{document}
%
\title{A New Probabilistic Algorithm for Approximate Model Counting}

\author{
Cunjing Ge$^{1,3}$,
Feifei Ma$^{1,2,3}$\thanks{Corresponding author},
Tian Liu$^4$,
Jian Zhang$^{1,3}$\\
}
\institute{
$^1$State Key Laboratory of Computer Science,\\
Institute of Software, Chinese Academy of Sciences\\
$^2$Laboratory of Parallel Software and Computational Science,\\
Institute of Software, Chinese Academy of Sciences \\
$^3$University of Chinese Academy of Sciences \\
$^4$School of Electronics Engineering and Computer Science, Peking University\\
}

\maketitle

\begin{abstract}
Constrained counting is important in domains ranging from artificial intelligence to software analysis. There are already a few approaches for counting models over various types of constraints. Recently, hashing-based approaches achieve both theoretical guarantees and scalability, but still rely on solution enumeration. In this paper, a new probabilistic polynomial time approximate model counter is proposed, which is also a hashing-based universal framework, but with only satisfiability queries. A variant with a dynamic stopping criterion is also presented. Empirical evaluation over benchmarks on propositional logic formulas and SMT(BV) formulas shows that the approach is promising.
\end{abstract}

\section{Introduction}
Constrained counting, the problem of counting the number of solutions for a set of constraints, is important in theoretical computer science and artificial intelligence. Its interesting applications in several fields include probabilistic inference~\cite{Roth96,ChaviraD08}, planning~\cite{DomshlakH07}, combinatorial designs and software engineering~\cite{LiuZ11,GeldenhuysDV12}. Constrained counting for propositional formulas is also called model counting, to which probabilistic inference is easily reducible. However, model counting is a canonical \#P-complete problem, even for polynomial-time solvable problems like 2-SAT~\cite{Valiant79}, thus presents fascinating challenges for both theoreticians and practitioners.

There are already a few approaches for counting solutions over propositional logic formulas and SMT(BV) formulas. Recently, hashing-based approximate counting achieves both strong theoretical guarantees and good scalability \cite{MCVF15}. The use of universal hash functions in counting problems began in \cite{Sipser83a,Stockmeyer83}, but the resulting algorithm scaled poorly in practice. A scalable approximate counter \texttt{ApproxMC} in \cite{ChakrabortyMV13}  scales to large problem instances, while preserves rigorous approximation guarantees. \texttt{ApproxMC} has been extended to finite-domain discrete integration, with applications to probabilistic inference \cite{ErmonGSS13,ChakrabortyFMSV14,BelleBP15}, and improved by designing efficient universal hash functions \cite{IvriiMMV16,ChakrabortyMMV16} and reducing the use of NP-oracle calls from linear to logarithmic \cite{ChakrabortyMV16b}.

The basic idea in \texttt{ApproxMC} is to estimate the model count by randomly and iteratively cutting the whole space down to
a small ``enough'' cell, using hash functions, and sampling it. The total model count is estimated by a multiplication of the number of solutions in this cell
and the ratio of whole space to the small cell. To determine the size of the small cell, which is essentially a small-scale model counting problem with the
model counts bounded by some thresholds, a model enumeration in the cell is adopted. In previous works, the enumeration query was handled by transforming it
into satisfiability queries, which is much more time-consuming than a single satisfiability query.
An algorithm called \texttt{MBound} \cite{GomesSS06b} only invokes satisfiability query once for each cut.
Its model count is determined with high precision by the number of cuts down to the boundary of being unsatisfiable.
However, this property is not strong enough to give rigorous guarantees,
and \texttt{MBound} only returns an approximation of upper or lower bound of the model count.



In this paper, a new hashing-based approximate counting algorithm, with only satisfiability query, is proposed, building on a correlation between the model count and the probability of the hashed formula being unsatisfiable. Dynamic stopping criterion for the algorithm to terminate, once meeting the theoretical guarantee of accuracy, are presented. Like previous works, a theoretical termination bound $T$ is shown. Theoretical insights over the efficiency of a prevalent heuristic strategy called leap-frogging are also provided.

The proposed approach is a general framework easy to handle various types of constraints. Prototype tools for propositional logic formulas and SMT(BV) formulas are implemented. An extensive evaluation on a suite of benchmarks demonstrates that (i) the approach significantly outperforms the state-of-the-art approximate model counters, including a counter designed for SMT(BV) formulas, and (ii) the dynamic stopping criterion is promising. The statistical results fit well with the theoretical bound of approximation, and it terminates much earlier before termination bound $T$ is met.

The rest of this paper is organized as follows.
Preliminary material is in Section~\ref{sect:prelim}, related works in Section~\ref{sect:relate}, the algorithm in Section~\ref{sect:approach}, analysis in Section~\ref{sect:analysis}, experimental results in Section~\ref{sect:eval}, and finally, concluding remarks in Section~\ref{sect:conclude}.

\section{Preliminaries}\label{sect:prelim}

Let $F(x)$ denote a propositional logic formula on $n$ variables $x = (x_1,\dots,x_n)$. Let $S$ and $S_F$ denote the whole space and the solution space of $F$, respectively. Let $\#F$ denote the cardinality of $S_F$, i.e. the number of solutions of $F$.

\noindent\textbf{$(\epsilon, \delta)$-bound} To count $\#F$, an $(\epsilon, \delta)$ approximation algorithm is an algorithm which on every input formula $F$, $\epsilon > 0$ and $\delta > 0$, outputs a number $\tilde{Y}$ such that $\Pr[(1+\epsilon)^{-1}\#F \le \tilde{Y} \le (1+\epsilon)\#F] \ge 1 - \delta$. These are called $(\epsilon, \delta)$-counters and $(\epsilon, \delta)$-bound, respectively \cite{KarpLM89}.

\noindent\textbf{Hash Function} Let $\mathcal{H}_F$ be a family of XOR-based bit-level hash functions on the variables of a formula $F$. Each hash function $H\in\mathcal{H}_F$ is of the form $H(x) = a_0 \bigoplus_{i=1}^{n} a_ix_i$, where $a_0,\dots,a_n$ are Boolean constants. In the hashing procedure \texttt{Hashing(F)}, a function $H\in\mathcal{H}_F$ is generated by independently and randomly choosing $a_i$s from a uniform distribution. Thus for an assignment $\alpha$, it holds that $\Pr_{H\in \mathcal{H}_F}(H(\alpha) = true) = \frac{1}{2}$. Given a formula $F$, let $F_i$ denote a hashed formula $F\land H_1\land\dots\land H_i$, where $H_1,\dots,H_i$ are independently generated by the hashing procedure.

\noindent\textbf{Satisfiability Query} Let \texttt{Solving(F)} denote the satisfiability query of a formula $F$. With a target formula $F$ as input, the satisfiability of $F$ is returned by \texttt{Solving(F)}.

\noindent\textbf{Enumeration Query} Let \texttt{Counting(F, p)} denote the bounded solution enumeration query. With a constraint formula $F$ and a threshold $p$ ($p \ge 2$) as inputs, a number $s$ is returned such that $s = \min(p - 1, \#F)$. Specifically, $0$ is returned for unsatisfiable $F$, or $p = 1$ which is meaningless.

\noindent\textbf{SMT(BV) Formula} SMT(BV) formulas are quantifier-free and fixed-size that combine propositional logic formulas with constraints of bit-vector theory. For example, $\neg(x + y = 0) \lor (x = y << 1)$, where $x$ and $y$ are bit-vector variables, $<<$ is the shift-left operator, is a propositional logic formula $\neg A \lor B$ that combines bit-vector constraints $A \equiv (x + y = 0)$ and $B \equiv (x = y << 1)$. To apply hash functions to an SMT(BV) formula, a bit-vector is treated as a set of Boolean variables.

\section{Related Works}\label{sect:relate}

\cite{BellareGP00} showed that almost uniform sampling from propositional constraints, a closely related problem to constrained counting, is solvable in probabilistic polynomial time with an NP oracle. Building on this, \cite{ChakrabortyMV13} proposed the first scalable approximate model counting algorithm \texttt{ApproxMC} for propositional formulas. \texttt{ApproxMC} is based on a family of 2-universal bit-level hash functions that compute XOR of randomly chosen propositional variables. Subsequently, \cite{ChakrabortyMMV16} presented an approximate model counter that uses word-level hash functions, which directly leverage the power of sophisticated SMT solvers, though the framework of the probabilistic algorithm is similar to \cite{ChakrabortyMV13}. In the current work, the family of hash functions in \cite{ChakrabortyMV13} is adopted, which was shown to be 3-independent in \cite{GomesSS06}, and is revealed to possess better properties than expected by the experimental results and the theoretical analysis in the current work.

\begin{algorithm}[!htbp]
\caption{}\label{alg:CMV}
\begin{algorithmic}[1]
\Function{ApproxMC($F$, $T$, $pivot$)}{}
	\For{$1$ \textbf{to} $T$}
		\State $c \leftarrow$ ApproxMCCore($F$, $p$, $pivot$)
		\State \textbf{if} ($c \ne 0$) \textbf{then} AddToList($C$, $c$)
	\EndFor
	\State \textbf{return} FindMedian($C$)
\EndFunction
\Function{ApproxMCCore($F$, $pivot$)}{}
	\State $F_0 \leftarrow F$
	\For{$i \leftarrow 0$ \textbf{to} $\infty$}
		\State $s \leftarrow$ Counting($F_i$, $pivot+1$)
		\State \textbf{if} ($0 \le s \le pivot$) \textbf{then} \textbf{return} $2^{i}s$
		\State $H_{i+1} \leftarrow$ Hashing($F$)
		\State $F_{i+1} \leftarrow F_i \land H_{i+1}$
	\EndFor
\EndFunction
\end{algorithmic}
\end{algorithm}

For completeness, \texttt{ApproxMC} \cite{ChakrabortyMV13,ChistikovDM15} is listed here as Algorithm~\ref{alg:CMV}. Its inputs are a formula $F$ and two accuracy parameters $T$ and $pivot$. $T$ determines the number of times \texttt{ApproxMCCore} invoked, and $pivot$ determines the threshold of the enumeration query. The function \texttt{ApproxMCCore} starts from the formula $F_0$, iteratively calls \texttt{Counting} and \texttt{Hashing} over each $F_i$, to cut the space (cell) of all models of $F_0$ using random hash functions, until the count of $F_i$ is no larger than $pivot$, then breaks out the loop and adds the approximation $2^{i}s$ into list $C$. The main procedure \texttt{ApproxMC} repeatedly invokes \texttt{ApproxMCCore} and collects the returned values, at last returning the median number of list $C$. The general algorithm in \cite{ChakrabortyMMV16} is similar to Algorithm~\ref{alg:CMV}, but could cut the cell with dynamically determined proportion instead of the constant $\frac{1}{2}$, due to the word-level hash functions. \cite{ChakrabortyMV16b} improves \texttt{ApproxMCCore} via  binary search to reduce the number of enumeration queries from linear to logarithmic. This binary search improvement is orthogonal to the proposed algorithm in the current work.

\section{Algorithm}\label{sect:approach}

In this section, a new hashing-based algorithm for approximate model counting, with only satisfiability queries, will be proposed, building on some probabilistic approximate correlations between the model count and the probability of the hashed formula being unsatisfiable.


Let $F_d=F\land H_1\land\dots\land H_d$ be a hashed formula resulted by iteratively hashing $d$ times independently over a formula $F$. $F_d$ is unsatisfiable if and only if no solution of $F$ satisfies $F_d$, thus $\Pr_{F_d}(F_d\ is\ unsat) = \Pr_{F_d}(F_d(\alpha) = false, \alpha\in S_F)$.
The following approximation is a key for the new algorithm
 \begin{eqnarray}
\Pr_{F_d}(F_d\ is\ unsat) \approx (1 - 2^{-d})^{\#F}. \label{eqs:ind_h}
\end{eqnarray}
Based on Equation~(\ref{eqs:ind_h}), an approximation of $\#F$ is achieved by taking logarithm on the value of $\Pr_{F_d}(F_d\ is\ unsat)$, which is estimated in turn by sampling $F_d$.
However, Equation~(\ref{eqs:ind_h}) can not be proven directly, since it is equivalent to
$\Pr_{F_d}(F_d(\alpha) = false,\alpha\in S_F)\approx  \prod_{\alpha\in S_F}\Pr_{F_d}(F_d(\alpha) = false)$,
but $\mathcal{H}_F$ was only known to be $3$-independent.

In the following, we consider a similar but different family of hash functions to provide some intuition (not proof) about Equation~(\ref{eqs:ind_h}).
From the definition $H(x) = a_0 \bigoplus_{i=1}^{n} a_ix_i$, it holds that (i) $\Pr_H(H(\alpha) = true) = \frac{1}{2}$ for any assignment $\alpha$, and (ii) $|S_H| = \frac{|S|}{2}$, unless $a_1 = \dots = a_n = 0$. Now let $\mathcal{G}$ be a family of hash functions also with these properties, defined as follows.
Each hash function $G$ in $\mathcal{G}$ is of the form
$$G(x) = \begin{cases} true& x\in S_G \\false & x\not\in S_G\end{cases}.$$
The solution set $S_G$ is generated by sampling $\frac{|S|}{2}$ points in $S$ without replacement (simple random sample).
For any given assignment $\alpha$, obviously  $\Pr_{G}(G(\alpha) = true) = \frac{1}{2}$.
Moreover, $\mathcal{G}$ is not $k$-independent for $k > \frac{|S|}{2}$, since the probability of more than $\frac{|S|}{2}$ variables having the same value is always zero.

\begin{theorem}\label{thm:ind_g}
Let $\hat{G}_d=G_1 \land \dots \land G_d$, where $G_1, \dots, G_d$ are independently sampled from $\mathcal{G}$, and let $F'_d=F \land \hat{G}_d$. Then
\begin{eqnarray}
\lim_{n \to \infty} \Pr_{F'_d}(F'_d\ is\ unsat) = (1-2^{-d})^{\#F}. \label{eqs:ind_g}
\end{eqnarray}
\end{theorem}

\begin{proof}
Since $\Pr_{G}(G(\alpha) = true) = \frac{1}{2}$ and $G_i$s are independent, we have $\Pr_{\hat{G}_d}(\hat{G}_d(\alpha) = true) = \left(\frac{1}{2}\right)^d$.
Let $S_{\hat{G}_d}$ denote the solution space of $\hat{G}_d$, then the expectation $E(|S_{\hat{G}_d}|)$ $=\left(\frac{1}{2}\right)^d|S| = 2^{n-d}$.
Since each $G_i$ is generated by simple random sample and uniquely determined by its solution set, the number of distinct functions of $\hat{G}_d$ is $\binom{|S|}{|S_{\hat{G}_d}|}$, and the number of distinct $\hat{G}_d$s which are unsatisfiable over $S_F$ is $\binom{|S-S_F|}{|S_{\hat{G}_d}|}$.
The probability of $S_F \cap S_{\hat{G}_d} = \emptyset$ is
\begin{eqnarray}
\frac{\binom{|S-S_F|}{|S_{\hat{G}_d}|}} {\binom{|S|}{|S_{\hat{G}_d}|}} & = & \frac{\binom{2^n-\#F}{2^{n-d}}}{\binom{2^n}{2^{n-d}}} = \frac{(2^n-\#F)!(2^n-2^{n-d})!}{(2^n)!(2^n-\#F-2^{n-d})!} \nonumber\\
& = & \frac{(2^n-2^{-d}2^n)}{2^n}\frac{(2^n-2^{-d}2^n-1)}{2^n-1}\dots \frac{(2^n-2^{-d}2^n-\#F+1)}{2^n-\#F+1}\label{eqs:binom}.
\end{eqnarray}
The LHS~(left-hand-side) of (\ref{eqs:binom}) is also $\Pr_{F'_d}(F'_d\ is\ unsat)$.
The RHS of (\ref{eqs:binom}) converges to $(1-2^{-d})^{\#F}$ as $n \to \infty$.
\qed
\end{proof}

Since $\mathcal{H}_F$ and $\mathcal{G}$ have some common characteristics, Equations (\ref{eqs:ind_g}) and (\ref{eqs:ind_h}) may also hold for $\mathcal{H}_F$.
In practice for finite $n$, Equation (\ref{eqs:ind_g}) only holds approximately. To make the value of $\Pr_{F'_d}(F'_d\ is\ unsat)$ arbitrarily close to $(1-2^{-d})^{\#F}$, dummy variables can be inserted into $F$ while keeping $\#F$ unchanged, such as inserting a new variable $x_{n+1}$ with constraint $x_{n+1} = true$. These variables and constraints are meaningless for SAT solvers and ignored at the beginning of the search. Experimental results indicate that Equation~(\ref{eqs:ind_h}) holds even without inserting any dummy variable.
Here we assume Equation~(\ref{eqs:ind_h}) holds in convenience of describing our approach.


\begin{algorithm}[!htbp]
\caption{Satisfiability Testing based Approximate Counter (STAC)}\label{alg:main}
\begin{algorithmic}[1]
\Function{STAC($F$, $T$)}{}
	\State initialize $C[i]$s with zeros
	\For{$t \leftarrow 1$ \textbf{to} $T$}
		\State $depth \leftarrow$ GetDepth($F$)
		\For{$i \leftarrow 0$ \textbf{to} $depth - 1$}
			\State $C[i] \leftarrow C[i] + 1$
		\EndFor
	\EndFor
	\State pick a number $d$ such that $C[d]$ is closest to $T/2$\label{alg:main:pick}
	\State $counter \leftarrow T - C[d]$
	\State \textbf{return} $\log_{1-2^{-d}}\frac{counter}{T}$\label{alg:main:return}\ \ \ /* return $0$ when $d=0$ */
\EndFunction
\Function{GetDepth($F$)}{}
	\State $F_0 \leftarrow F$
	\For{$i \leftarrow 0$ \textbf{to} $\infty$}\label{alg:depth:beg}
		\State $b \leftarrow$ Solving($F_i$)
		\State \textbf{if} ($b$ is false) \textbf{then} \textbf{return} $i$
		\State $H_{i+1} \leftarrow$ Hashing($F_i$)
		\State $F_{i+1} \leftarrow F_i \land H_{i+1}$
	\EndFor\label{alg:depth:end}
\EndFunction
\end{algorithmic}
\end{algorithm}

The pseudo-code for our approach is presented in Algorithm~\ref{alg:main}. The inputs are the target formula $F$ and a constant $T$ which determines the number of times \texttt{GetDepth} invoked. \texttt{GetDepth} calls \texttt{Solving} and \texttt{Hashing} repeatedly until an unsatisfiable formula $F_{depth}$ is encountered, and returns the $depth$. Every time \texttt{GetDepth} returns a $depth$, the value of $C[i]$ is increased, for all $i < depth$. At line \ref{alg:main:pick}, the algorithm picks a number $d$ such that $C[d]$ is close to $T/2$, since the error estimation fails when $C[d]/T$ is close to 0 or 1. The final counting result is returned by the formula $\log_{1-(1/2)^d}\frac{counter}{T}$ at line \ref{alg:main:return}. Theoretical analysis of the value of $T$ and the correctness of algorithm are in Section~\ref{sect:analysis}.

\paragraph{Dynamic Stopping Criterion} The essence of Algorithm~\ref{alg:main} is a randomized sampler over a binomial distribution. The number of samples is determined by the value of $T$, which is pre-computed for a given $(\epsilon, \delta)$-bound, and we loosen the bound of $T$ to meet the guarantee in theoretical analysis. However, it usually does not loop $T$ times in practice. A variation with dynamic stopping criterion is presented in Algorithm~\ref{alg:dynstop}.

Line \ref{alg:dynstop:beg} to \ref{alg:dynstop:sameend} is the same as Algorithm~\ref{alg:main}, still setting $T$ as a stopping rule and terminating whenever $t = T$. Line \ref{alg:dynstop:kbeg} to \ref{alg:dynstop:kend} is the key part of this variation, calculating the binomial proportion confidence interval $[L, U]$ of an intermediate result $M$ for each cycle. A commonly used formula $q \pm z_{1-\delta} \sqrt{\frac{q(1-q)}{t}}$~\cite{BrownCD01,Wallis13} is adopted, which is justified by the central limit theorem to compute the $1-\delta$ confidence interval. The exact count $\#F$ lies in the interval $[L, U]$ with probability $1-\delta$. Combining the inequalities presented in line \ref{alg:dynstop:kif}, the interval $[(1+\epsilon)^{-1}M, (1+\epsilon)M]$ is the $(\epsilon, \delta)$-bound. The algorithm terminates when the condition in line \ref{alg:dynstop:kif} comes true, and its correctness is obvious. The time complexity of Algorithm~\ref{alg:dynstop} is still the same as the original algorithm, though it usually terminates earlier.

\begin{algorithm}[!htbp]
\caption{STAC with Dynamic Stopping Criterion}\label{alg:dynstop}
\begin{algorithmic}[1]
\Function{STAC\_DSC($F$, $T$, $\epsilon$, $\delta$)}{}
	\State initialize $C[i]$s with zeros\label{alg:dynstop:beg}
	\For{$t \leftarrow 1$ \textbf{to} $T$}
		\State $depth \leftarrow$ GetDepth($F$)
		\For{$i \leftarrow 0$ \textbf{to} $depth - 1$}
			\State $C[i] \leftarrow C[i] + 1$
		\EndFor\label{alg:dynstop:sameend}
		\For {\textbf{each} $d$ \textbf{that} $C[d] > 0$}\label{alg:dynstop:kbeg}
			\State $q \leftarrow \frac{t - C[d]}{t}$
			\State $M \leftarrow \log_{1-2^{-d}} q$
			\State $U \leftarrow \log_{1-2^{-d}} (q - z_{1-\delta} \sqrt{\frac{q(1-q)}{t}})$
			\State $L \leftarrow \log_{1-2^{-d}} (q + z_{1-\delta} \sqrt{\frac{q(1-q)}{t}})$		
			\If{$U < (1+\epsilon)M$ and $L > (1+\epsilon)^{-1}M$}\label{alg:dynstop:kif}
				\State \textbf{return} $M$
			\EndIf
		\EndFor\label{alg:dynstop:kend}
	\EndFor
\EndFunction
\end{algorithmic}
\end{algorithm}

\paragraph{Satisfiability And Enumeration Query}
For a propositional logic formula or an SMT formula, there is a direct way to enumerate solutions with satisfiability queries. For example, we assert the negation of a solution which is found by satisfiability query, so that the constraint solver will provide other models. So \texttt{Counting(F, p)} invokes up to $p$ times \texttt{Solving(F)} by this way. This method is adopted by all previous hashing-based $(\epsilon, \delta)$-counters \cite{ChakrabortyMV13,ChistikovDM15,ChakrabortyMMV16}.

\paragraph{Leap-frogging Strategy}
Recall that \texttt{GetDepth} is invoked $T$ times with the same arguments, and the loop of line \ref{alg:depth:beg} to \ref{alg:depth:end} in the pseudo-code of \texttt{GetDepth} in Algorithm~\ref{alg:main} is time consuming for large $i$. A heuristic called leap-frogging to overcome this bottleneck was proposed in \cite{ChakrabortyMV13CAV,ChakrabortyMV13}. Their experiments indicate that this strategy is extremely efficient in practice. The average depth $\bar{d}$ of each invocation of \texttt{GetDepth} is recorded. In all subsequent invocations, the loop starts by initializing $i$ to $\bar{d}-\mathrm{offset}$, and sets $i$ to $\bar{d}-2*\mathrm{offset}$ for unsatisfiable $F_i$ repeatedly until a proper $i$ is found for a satisfiable $F_i$. In practice, the constant offset is usually set to $5$. Theorem~\ref{rem:leap} in Section~\ref{sect:analysis} shows that the $depth$ computed by \texttt{GetDepth} lies in an interval $[d, d+7]$ with probability over 90\%. So it suffices to invoke \texttt{Solving} in constant time since the second iteration.

\paragraph{Wilson Score Interval}
The central limit theorem applies poorly to binomial distributions for small sample size or proportion close to 0 or 1. The following interval \cite{Wilson1927}
$$\frac{1}{1+\frac{z^2}{t}}[q+\frac{z^2}{2t} \pm z\sqrt{\frac{q(1-q)}{t} + \frac{z^2}{4t^2}}]$$
has good properties to achieve better approximations. Note that the overhead of computing intervals is negligible.

\section{Analysis}\label{sect:analysis}
\setcounter{theorem}{1}
In this section, we assume Equation~(\ref{eqs:ind_h}) holds. Based on this assumption, theoretical results on the error estimation of our approach are presented.
Recall that in Algorithm~\ref{alg:main}, $\#F$ is approximated by a value $\log_{1-2^{-d}}\frac{counter}{T}$, based on Equation~(\ref{eqs:ind_h}).
Let $q_d$ denote the value of $(1-2^{-d})^{\#F}$. We will assume that $\Pr(F_d\ is\ unsat) = q_d$ for a randomly generated formula $F_d$.
This assumption is justified by Equations~(\ref{eqs:ind_h}) and (\ref{eqs:ind_g}).
Then the $counter$ in Algorithm~\ref{alg:main} is a random variable following a binomial distribution $\mathbb{B}(T, q_d)$.
Since the ratio $\frac{counter}{T}$ is the proportion of successes in a Bernoulli trial process, it is used to estimate the value of $q_d$.

\begin{theorem}\label{rem:errest}
Let $z_{1-\delta}$ be the $1-\delta$ quantile of $\mathbb{N}(0,1)$ and
\begin{eqnarray}
T = max\left(\lceil (\frac{z_{1-\delta}}{2 q_d(1 - q_d^\epsilon)})^2\rceil, \lceil(\frac{z_{1-\delta}}{2(q_d^{(1+\epsilon)^{-1}} - q_d)})^2) \rceil\right). \label{eqs:t}
\end{eqnarray}
Then $\Pr[\frac{\#F}{1+\epsilon} \le \log_{1-2^{-d}}\frac{counter}{T} \le (1+\epsilon)\#F] \ge 1-\delta$.
\end{theorem}
\begin{proof}
By above discussions, the ratio $\frac{counter}{T}$ is the proportion of successes in a Bernoulli trial process which follows the distribution $\mathbb{B}(T, q_d)$. Then we use the approximate formula of a binomial proportion confidence interval $q_d \pm z_{1-\delta} \sqrt{\frac{q_d(1 - q_d)}{T}}$, i.e., $\Pr[q_d - z_{1-\delta} \sqrt{\frac{q_d (1 - q_d)}{T}} \le \frac{counter}{T} \le q_d + z_{1-\delta} \sqrt{\frac{q_d (1 - q_d)}{T}}] \ge 1-\delta$.
The $\log$ function is monotone, so we only have to consider the following two inequalities:
\begin{eqnarray}
\log_{1-2^{-d}}{(q_d - z_{1-\delta} \sqrt{\frac{q_d (1 - q_d)}{T}})} \le (1+\epsilon)\#F, \label{eqs:ineq1} \\
(1+\epsilon)^{-1}\#F \le \log_{1-2^{-d}}{(q_d + z_{1-\delta} \sqrt{\frac{q_d (1 - q_d)}{T}})}. \label{eqs:ineq2}
\end{eqnarray}

We first consider Equation~(\ref{eqs:ineq1}). By substituting $\log_{1-2^{-d}}{q_d}$ for $\#F$, we have
\begin{eqnarray*}
& & \log_{1-2^{-d}}{(q_d - z_{1-\delta} \sqrt{\frac{q_d (1 - q_d)}{T}})} \le (1+\epsilon)\log_{1-2^{-d}}{q_d} \\
& \Leftrightarrow & q_d - z_{1-\delta} \sqrt{\frac{q_d (1 - q_d)}{T}} \ge q_d^{(1+\epsilon)} \\
& \Leftrightarrow & q_d (1 - q_d^{\epsilon}) \ge z_{1-\delta} \sqrt{\frac{q_d (1 - q_d)}{T}} \\
& \Leftrightarrow & T \ge (\frac{z_{1-\delta}}{q_d (1 - q_d^{\epsilon})})^2 q_d (1- q_d).
\end{eqnarray*}

Since $0 \le q_d \le 1$, we have $\sqrt{q_d (1 - q_d)} \le \frac{1}{2}$. Therefore, $T = \lceil (\frac{z_{1-\delta}}{2 q_d(1 - q_d^\epsilon)})^2 \rceil \ge (\frac{z_{1-\delta}}{q_d (1 - q_d^{\epsilon})})^2 q_d (1 - q_d)$.

We next consider Equation~(\ref{eqs:ineq2}). Similarly, we have
\begin{eqnarray*}
& & \log_{1-2^{-d}}{(q_d + z_{1-\delta} \sqrt{\frac{q_d (1 - q_d)}{T}})} \ge (1+\epsilon)^{-1}\log_{1-2^{-d}}{q_d} \\
& \Leftrightarrow & q_d + z_{1-\delta} \sqrt{\frac{q_d (1 - q_d)}{T}} \le q_d^{1/(1+\epsilon)} \\
& \Leftrightarrow & T \ge (\frac{z_{1-\delta}}{q_d^{1/(1+\epsilon)} - q_d })^2 q_d (1- q_d).
\end{eqnarray*}
So Equation~(\ref{eqs:t}) implies Equations~(\ref{eqs:ineq1}) and~(\ref{eqs:ineq2}).
\end{proof}

Theorem~\ref{rem:errest} shows that the result of Algorithm~\ref{alg:main} lies in the interval $[(1+\epsilon)^{-1}\#F, (1+\epsilon)\#F]$ with probability at least $1-\delta$ when $T$ is set to a proper value. So we focus on the possible smallest value of $T$ in subsequent analysis.

The next two lemmas are easy to show by derivations.
\begin{lemma}\label{lem:mono1}
$\frac{z_{1-\delta}}{2 x (1 - x^\epsilon)}$ is monotone increasing and monotone decreasing in $[(1+\epsilon)^{-\frac{1}{\epsilon}}, 1]$ and $[0, (1+\epsilon)^{-\frac{1}{\epsilon}}]$ respectively.
\end{lemma}
\begin{lemma}\label{lem:mono2}
$\frac{z_{1-\delta}}{2(x^{1/(1+\epsilon)} - x)}$ is monotone increasing and monotone decreasing in $[(1+\epsilon)^{-\frac{1+\epsilon}{\epsilon}}, 1]$ and $[0, (1+\epsilon)^{-\frac{1+\epsilon}{\epsilon}}]$ respectively.
\end{lemma}

\begin{theorem}\label{rem:qinterval}
If $\#F > 5$, then there exists a proper integer value of $d$ such that $q_d \in [0.4, 0.65]$.
\end{theorem}
\begin{proof}
Let $x$ denote the value of $q_d = (1-\frac{1}{2^d})^{\#F}$, then we have $(1-\frac{1}{2^{d+1}})^{\#F} = (\frac{1}{2}+\frac{x^\frac{1}{\#F}}{2})^{\#F}$.
Consider the derivation
\begin{eqnarray*}
& \frac{d}{d\#F}(\frac{1}{2}+\frac{x^\frac{1}{\#F}}{2})^{\#F} & = (\frac{1}{2}+\frac{x^\frac{1}{\#F}}{2})^{\#F} \ln{(\frac{1}{2}+\frac{x^\frac{1}{\#F}}{2})} \frac{x^{\frac{1}{\#F}}}{2} \ln{x} \frac{d}{d\#F}(\#F^{-1}).
\end{eqnarray*}
Note that $(\frac{1}{2}+\frac{x^\frac{1}{\#F}}{2})^{\#F}$ and $\frac{x^{\frac{1}{\#F}}}{2}$ are the positive terms and $\ln{(\frac{1}{2}+\frac{x^\frac{1}{\#F}}{2})}$, $\ln{x}$ and $\frac{d}{d\#F}(\#F^{-1})$ are the negative terms. Therefore, the derivation is negative, i.e., $(\frac{1}{2}+\frac{x^\frac{1}{\#F}}{2})^{\#F}$ is monotone decreasing with respect to $\#F$. In addition, $(\frac{1}{2}+\frac{x^\frac{1}{5}}{2})^{5}$ is the upper bound when $\#F \ge 5$.

Let $x = 0.4$, then $(1-\frac{1}{2^{d+1}})^{\#F} \le (\frac{1}{2}+\frac{0.4^\frac{1}{5}}{2})^{5} \approx 0.65$. Since $(1-\frac{1}{2^{0}})^{\#F} = 0$ and $\lim_{d \to +\infty} (1-\frac{1}{2^{d}})^{\#F} = 1$ and $(1-\frac{1}{2^{d}})^{\#F}$ is continuous with respect to $d$, we consider the circumstances close to the interval $[0.4, 0.65]$. Assume there exists an integer $\sigma$ such that $(1-\frac{1}{2^{\sigma}})^{\#F} < 0.4$ and $(1-\frac{1}{2^{\sigma+1}})^{\#F} > 0.65$. According to the intermediate value theorem, we can find a value $e > 0$ such that $(1-\frac{1}{2^{\sigma + e}})^{\#F} = 0.4$. Obviously, we have $(1-\frac{1}{2^{\sigma + e + 1}})^{\#F} \le 0.65$ which is contrary with the monotone decreasing property.
\end{proof}

From Theorem~\ref{rem:qinterval} and Lemma~\ref{lem:mono1} and~\ref{lem:mono2}, it suffices to consider the results of Equation~(\ref{eqs:t}) when $q_d = 0.4$ and $q_d = 0.65$. For example, $T = 22$ for $\epsilon = 0.8$ and $\delta = 0.2$, $T = 998$ for $\epsilon = 0.1$ and $\delta = 0.1$, etc.
We therefore pre-computed a table of the value of $T$. The proof of next theorem is omitted.


\begin{theorem}\label{rem:leap}
There exists an integer $d$ such that $q_d < 0.05$ and $q_{d+7} > 0.95$.
\end{theorem}

Let $depth$ denote the result of \texttt{GetDepth} in Algorithm~\ref{alg:main}. Then $F_d$ is unsatisfiable only if $d \ge depth$.
Theorem~\ref{rem:leap} shows that there exists an integer $d$ such that $\Pr(depth < d) < 0.05$ and $\Pr(depth < d+7) > 0.95$, i.e., $\Pr(d \le depth \le d+7) > 0.9$. So in most cases, the value of $depth$ lies in an interval $[d, d+7]$. Also, it is easy to see that $\log_2 \#F$ lies in this interval as well.
The following theorem is obvious now.

\begin{theorem}\label{rem:complexity}
Algorithm~\ref{alg:main} runs in time linear in $\log_2 \#F$ relative to an NP-oracle. 
\end{theorem}


\section{Evaluation}\label{sect:eval}

To evaluate the performance and effectiveness of our approach, two prototype implementations \texttt{STAC\_CNF} and \texttt{STAC\_BV} with dynamic stopping criterion for propositional logic formulas and SMT(BV) formulas are built respectively.
We considered a wide range of benchmarks from different domains: grid networks, plan recognition, DQMR networks, Langford sequences, circuit synthesis,
random 3-CNF, logistics problems and program synthesis~\cite{SangBK05,KrocSS11,ChakrabortyMV13,ChakrabortyMMV16}. For lack of space, we present results for only for a subset of the benchmarks.
All our experiments were conducted on a single core of an Intel Xeon 2.40GHz (16 cores) machine with 32GB memory and CentOS6.5 operating system.

\subsection{Quality of Approximation}

Recall that our approach is based on Equation~(\ref{eqs:ind_h}) which has not been proved.
So we would like to see whether the approximation fits the bound.
We experimented 100 times on each instance.

\begin{table*}[!htbp]
\centering
\caption{Statistical results of 100-times experiments on \texttt{STAC\_CNF} ($\epsilon = 0.8, \delta = 0.2$)}\label{table:stat_cnf82}
\begin{tabular}{p{80pt}p{30pt}p{40pt}p{86pt}p{21pt}p{21pt}p{21pt}p{21pt}}
\toprule
Instance 			& $n$	& $\#F$				& $[1.8^{-1}\#F, 1.8\#F]$					& Freq. & $\bar{t}$ (s)	& $\bar{T}$	& $\bar{Q}$	\\
\midrule
special-1			& 20	& $1.0\times 10^6$	& $[5.8\times 10^5, 1.9\times 10^6]$		& 82	& 0.3			& 12.2		& 86.7	\\
special-2			& 20	& $1$				& $[0.6, 1.8]$								& 86	& 0.6			& 12.6		& 37.6	\\
special-3			& 25	& $3.4\times 10^7$	& $[1.9\times 10^7, 6.0\times 10^7]$		& 82	& 11.2			& 11.8		& 90.1	\\
5step				& 177	& $8.1\times 10^4$	& $[4.5\times 10^4, 1.5\times 10^5]$		& 90 	& 0.1			& 11.9		& 80.5	\\
blockmap\_05\_01	& 1411	& $6.4\times 10^2$	& $[3.6\times 10^2, 1.2\times 10^3]$		& 84	& 1.1			& 12.0		& 73.8	\\
blockmap\_05\_02	& 1738	& $9.4\times 10^6$	& $[5.2\times 10^6, 1.7\times 10^7]$		& 89	& 12.7			& 11.8		& 87.7	\\
blockmap\_10\_01	& 11328	& $2.9\times 10^6$	& $[1.6\times 10^6, 5.2\times 10^6]$		& 83	& 80.3			& 12.0		& 85.0	\\
fs-01				& 32	& $7.7\times 10^2$	& $[4.3\times 10^2, 1.4\times 10^3]$		& 80	& 0.02			& 12.6		& 76.2	\\
or-50-10-10-UC-20	& 100	& $3.7\times 10^6$	& $[2.0\times 10^6, 6.6\times 10^6]$		& 77	& 7.7			& 12.0		& 86.1	\\
or-60-10-10-UC-40	& 120	& $3.4\times 10^6$	& $[1.9\times 10^6, 6.1\times 10^6]$		& 91	& 3.5			& 12.1		& 86.0	\\
\bottomrule
\end{tabular}
\vspace{-1ex}
\end{table*}

In Table~\ref{table:stat_cnf82}, column 1 gives the instance name,
column 2 the number of Boolean variables $n$,
column 3 the exact counts $\#F$, and column 4 the interval $[1.8^{-1}\#F, 1.8\#F]$.
The frequencies of approximations that lie in the interval $[1.8^{-1}\#F, 1.8\#F]$ in 100 times of experiments are presented in column 5.
The average time consumptions, average number of iterations, and average number of SAT query invocations are presented in column 6, 7 and 8 respectively,
which also indicate the advantages of our approach.


Under the dynamic stopping criterion, the counts returned by our approach should lie in an interval $[1.8^{-1}\#F, 1.8\#F]$ with probability $80\%$ for $\epsilon = 0.8$ and $\delta = 0.2$. The statistical results in Table~\ref{table:stat_cnf82} show that the frequencies are around $80$ for 100-times experiments which fit the $80\%$ probability. The average number of iterations $\bar{T}$ listed in Table~\ref{table:stat_cnf82} is smaller than the theoretical termination bound $T$ which is $22$, indicating that the dynamic stopping technique significantly improves the efficiency. In addition, the values of $\bar{T}$ appear to be stable for different instances, hinting that there exists a constant upper bound on $T$ which is irrelevant to instances.

In Section~\ref{sect:approach}, we considered a similar but different function family $\mathcal{G}$ and proved Equation~(\ref{eqs:ind_g}).
It suggests that Equation~(\ref{eqs:ind_h}) may hold on $\mathcal{H}$ for infinite $n$.
However, our approach does not insert any dummy variables to increase $n$.
Intuitively, our approach may start to fail on ``loose'' formulas, i.e., with an ``infinitesimal'' fraction of non-models.
Instance \texttt{special-1} and \texttt{special-3} are such ``loose'' formulas
where \texttt{special-1} has $2^{20}$ models with only 20 variables and \texttt{special-3} has $2^{25}-1$ models with 25 variables.
Instance \texttt{special-2} is another extreme case which only has one model.
The results in Table~\ref{table:stat_cnf82} demonstrate that \texttt{STAC\_CNF} works fine even on these extreme cases.

\begin{table*}[!htbp]
\centering
\caption{Statistical results of 100-times experiments on \texttt{STAC\_CNF} ($\epsilon = 0.2, \delta = 0.1$)}\label{table:stat_cnf21}
\begin{tabular}{p{80pt}p{30pt}p{40pt}p{86pt}p{21pt}p{21pt}p{21pt}p{21pt}}
\toprule
Instance 			& $n$	& $\#F$				& $[1.2^{-1}\#F, 1.2\#F]$					& Freq. & $\bar{t}$ (s)	& $\bar{T}$	& $\bar{Q}$	\\
\midrule
special-1			& 20	& $1.0\times 10^6$	& $[8.7\times 10^5, 1.3\times 10^6]$		& 86	& 4.0			& 179		& 1023	\\
special-2			& 20	& $1$				& $[0.8, 1.2]$								& 91	& 0.1			& 179		& 540	\\
special-3			& 25	& $3.4\times 10^7$	& $[2.8\times 10^7, 4.0\times 10^7]$		& 91	& 138			& 178		& 1029	\\
5step				& 177	& $8.1\times 10^4$	& $[6.8\times 10^4, 9.8\times 10^5]$		& 96 	& 1.9			& 190		& 1078	\\
blockmap\_05\_01	& 1411	& $6.4\times 10^2$	& $[5.3\times 10^2, 7.7\times 10^2]$		& 94	& 17.1			& 190		& 1069	\\
blockmap\_05\_02	& 1738	& $9.4\times 10^6$	& $[7.9\times 10^6, 1.1\times 10^7]$		& 87	& 281			& 193		& 1088	\\
blockmap\_10\_01	& 11328	& $2.9\times 10^6$	& $[2.4\times 10^6, 3.5\times 10^6]$		& 93	& 1371			& 180		& 1034	\\
fs-01				& 32	& $7.7\times 10^2$	& $[6.4\times 10^2, 9.2\times 10^2]$		& 91	& 0.1			& 172		& 975	\\
or-50-10-10-UC-20	& 100	& $3.7\times 10^6$	& $[3.1\times 10^6, 4.4\times 10^6]$		& 90	& 140			& 166		& 925	\\
or-60-10-10-UC-40	& 120	& $3.4\times 10^6$	& $[2.8\times 10^6, 4.1\times 10^6]$		& 92	& 66			& 167		& 949	\\
\bottomrule
\end{tabular}
\vspace{-1ex}
\end{table*}

We also considered another pair of parameters $\epsilon = 0.2, \delta = 0.1$.
Then the interval should be $[1.2^{-1}\#F, 1.2\#F]$ and the probability should be $90\%$.
Table~\ref{table:stat_cnf21} shows the results on such parameter setting.
The frequencies that the approximation lies in interval $[1.2^{-1}\#F, 1.2\#F]$ are all around or over 90 which fits the $90\%$ probability.

\begin{table*}[!htbp]
\centering
\caption{Statistical results of 100-times experiments on \texttt{STAC\_BV} ($\epsilon = 0.8, \delta = 0.2$)}\label{table:stat_bv}
\begin{tabular}{p{73pt}p{22pt}p{45pt}p{90pt}p{25pt}p{25pt}p{25pt}p{20pt}}
\toprule
Instance 		& TB.	& $\#F$				& $[1.8^{-1}\#F, 1.8\#F]$				& Freq. & $\bar{t}$ (s)	& $\bar{T}$	& $\bar{Q}$	\\
\midrule
FINDpath1		& 32	& $4.1\times 10^6$	& $[2.3\times 10^6, 7.3\times 10^6]$	& 83 	& 27.5	& 12.4		& 88.0	\\
queue			& 16	& $8.4\times 10$	& $[4.7\times 10, 1.5\times 10^2]$		& 75 	& 1.7	& 12.0		& 70.6	\\
getopPath2		& 24	& $8.1\times 10^3$	& $[4.5\times 10^3, 1.5\times 10^4]$	& 88	& 2.7	& 12.2		& 79.5	\\
coloring\_4		& 32	& $1.8\times 10^9$	& $[1.0\times 10^9, 3.3\times 10^9]$	& 76	& 51.9	& 12.0		& 96.1	\\
FISCHER2-7-fair	& 240	& $3.0\times 10^4$	& $[1.7\times 10^4, 5.4\times 10^4]$	& 79	& 149	& 11.8		& 79.8	\\
case2			& 24	& $4.2\times 10^6$	& $[2.3\times 10^6, 7.6\times 10^6]$	& 79	& 16.5	& 12.4		& 89.3	\\
case4			& 16	& $3.3\times 10^4$	& $[1.8\times 10^4, 5.9\times 10^4]$	& 87	& 2.2	& 12.5		& 85.2	\\
case7			& 18	& $1.3\times 10^5$	& $[7.3\times 10^4, 2.4\times 10^5]$	& 83	& 2.9	& 12.4		& 84.1	\\
case8			& 24	& $8.4\times 10^6$	& $[4.7\times 10^6, 1.5\times 10^7]$	& 82	& 14.4	& 12.1		& 91.1	\\
case11			& 15	& $1.6\times 10^4$	& $[9.1\times 10^3, 2.9\times 10^4]$	& 76	& 2.1	& 12.0		& 81.2	\\
\bottomrule
\end{tabular}
\vspace{-1ex}
\end{table*}

Table~\ref{table:stat_bv} similarly shows the results of 100-times experiments on \texttt{STAC\_BV}.
Its column 2 gives the sum of widths of all bit-vector variables (Boolean variable is counted as a bit-vector of width 1) instead.
The statistical results demonstrate that the dynamic stopping criterion is also promising on SMT(BV) problems.

\subsection{Performance Comparison with $(\epsilon, \delta)$-counters}

We compared our tools with \texttt{ApproxMC2} \cite{ChakrabortyMV16b} and \texttt{SMTApproxMC} \cite{ChakrabortyMMV16} which are hashing-based $(\epsilon, \delta)$-counters.
Both \texttt{STAC\_CNF} and \texttt{ApproxMC2} use \texttt{CryptoMiniSAT}~\cite{SoosNC09}, an efficient SAT solver designed for XOR clauses.
\texttt{STAC\_BV} and \texttt{SMTApproxMC} use the state-of-the-art SMT(BV) solver \texttt{Boolector}~\cite{BrummayerB09}.

\begin{figure*}[htbp]
\centering
\begin{minipage}[t]{0.45\textwidth}
\centering
\includegraphics[width=\textwidth]{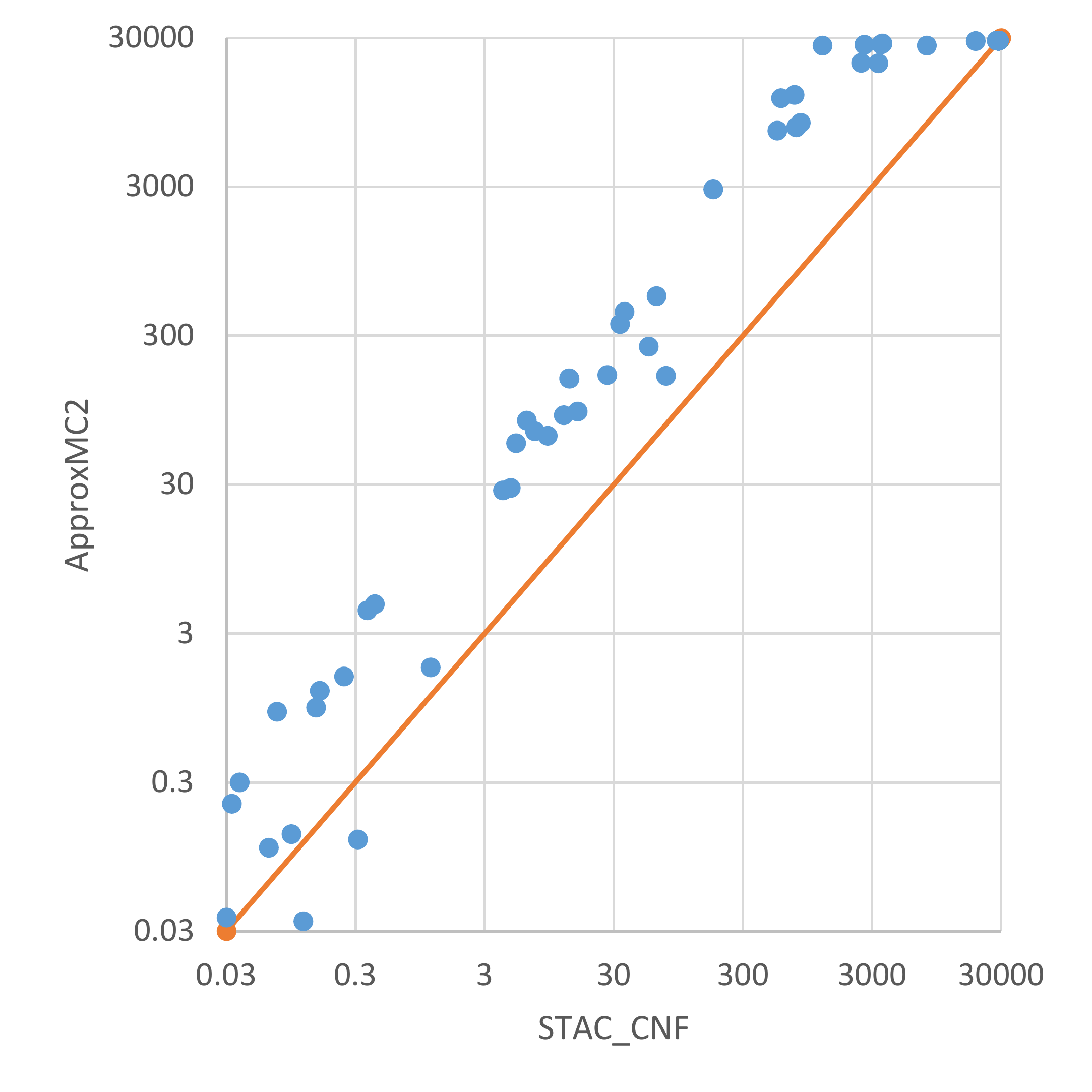}
\vspace{-6ex}
\caption{Performance comparison between \texttt{STAC\_CNF} and \texttt{ApproxMC2} with $\epsilon = 0.8, \delta = 0.2$ and 8 hours timeout}\label{fig:comp_cnf}
\end{minipage}
\hspace{2ex}
\begin{minipage}[t]{0.45\textwidth}
\centering
\includegraphics[width=\textwidth]{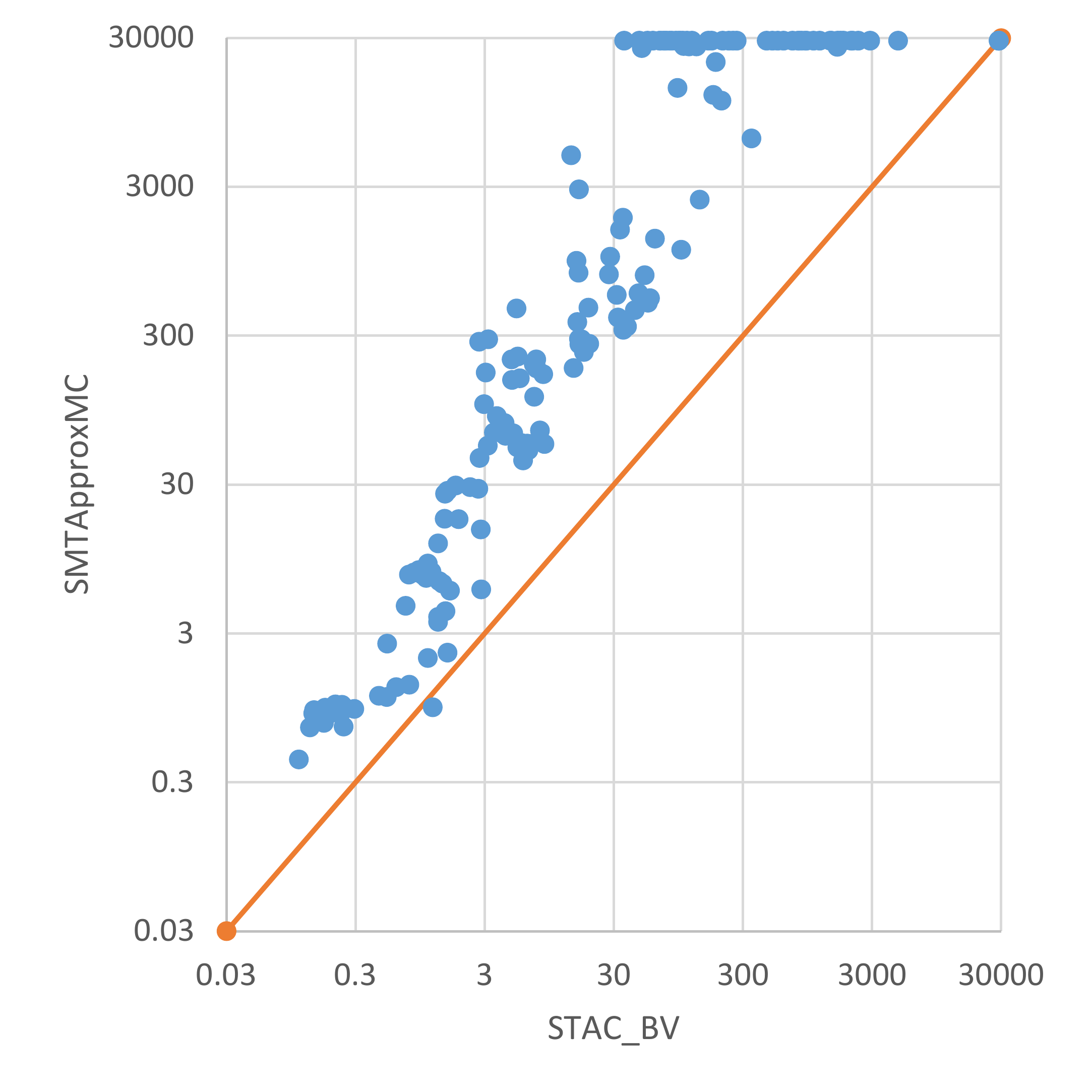}
\vspace{-6ex}
\caption{Performance comparison between \texttt{STAC\_BV} and \texttt{SMTApproxMC} with $\epsilon = 0.8, \delta = 0.2$ and 8 hours timeout}\label{fig:comp_bv}
\end{minipage}
\vspace{-1ex}
\end{figure*}

We first conducted experiments with $\epsilon = 0.8, \delta = 0.2$ which are also used in evaluation in previous works \cite{ChakrabortyMV16b,ChakrabortyMMV16}.
Figure~\ref{fig:comp_cnf} presents a comparison on performance between \texttt{STAC\_CNF} and \texttt{ApproxMC2}.
Each point represents an instance, whose  $x$-coordinate and $y$-coordinate are the running times of \texttt{STAC\_CNF} and \texttt{ApproxMC2} on this instance, respectively.
The figure is in logarithmic coordinates and demonstrates that \texttt{STAC\_CNF} outperforms \texttt{ApprxMC2} by about one order of magnitude.
Figure~\ref{fig:comp_bv} presents a similar comparison on performance between \texttt{STAC\_BV} and \texttt{SMTApproxMC},
showing that \texttt{STAC\_BV} outperforms \texttt{SMTApproxMC} by one or two order of magnitude.
Furthermore, the advantage enlarges as the scale grows.

\begin{table*}[!htbp]
\centering
\caption{Performance comparison between \texttt{STAC\_CNF} and \texttt{ApproxMC2} with different pairs of $(\epsilon, \delta)$ parameters}\label{table:comp_param}
\begin{tabular}{p{45pt}p{55pt}p{26pt}p{26pt}p{26pt}p{2pt}p{26pt}p{26pt}p{26pt}p{26pt}p{26pt}}
\toprule
\multicolumn{2}{c}{\multirow{2}{*}{\backslashbox[70pt]{$(\epsilon, \delta)$}{Instance}}}	& \multicolumn{3}{c}{blockmap}	&	
	& \multirow{2}{*}{fs-01}	& \multirow{2}{*}{5step}	& \multirow{2}{*}{ran5}	& \multirow{2}{*}{ran6}	& \multirow{2}{*}{ran7}	\\
\cline{3-5}
	&	& 05\_01	& 05\_02	& 10\_01	\\
\midrule
\multirow{2}{*}{$(0.8, 0.3)$}
	& Time Ratio	& 1.11	& 3.99	& 1.22	&	& 3.00	& 3.83	& 6.53	& 8.24	& 5.57	\\
	& \#Calls Ratio	& 22.60	& 39.02	& 17.91	&	& 19.12	& 23.11	& 22.53	& 21.28	& 23.68	\\
\hline
\multirow{2}{*}{$(0.8, 0.2)$}	
	& Time Ratio	& 1.84	& 6.16	& 2.44	&	& 2.80	& 6.05	& 9.61	& 15.41	& 7.37	\\
	& \#Calls Ratio	& 26.70	& 34.68	& 25.16	&	& 33.46	& 27.24	& 33.35	& 38.22	& 30.94	\\
\hline
\multirow{2}{*}{$(0.8, 0.1)$}	
	& Time Ratio	& 2.27	& 7.36	& 3.72	&	& 5.25	& 12.62	& 9.60	& 9.54	& 8.19	\\
	& \#Calls Ratio	& 44.88	& 48.26	& 40.01	&	& 49.40	& 43.03	& 46.12	& 44.84	& 52.63	\\
\hline
\multirow{2}{*}{$(0.4, 0.3)$}	
	& Time Ratio	& 0.75	& 1.37	& 0.42	& 	& 3.00	& 5.04	& 1.97	& 2.31	& 2.74	\\
	& \#Calls Ratio	& 17.75	& 36.20	& 14.69	& 	& 16.40	& 27.63	& 21.07	& 27.34	& 21.63	\\
\hline
\multirow{2}{*}{$(0.4, 0.2)$}	
	& Time Ratio	& 0.77	& 1.44	& 0.86	& 	& 4.50	& 7.70	& 2.82	& 1.77	& 3.02	\\
	& \#Calls Ratio	& 20.91	& 26.35	& 29.16	& 	& 26.72	& 40.66	& 26.49	& 27.82	& 28.94	\\
\hline
\multirow{2}{*}{$(0.4, 0.1)$}	
	& Time Ratio	& 1.08	& 2.57	& 1.29	& 	& 4.90	& 7.09	& 3.84	& 3.43	& 3.11	\\
	& \#Calls Ratio	& 37.16	& 46.28	& 39.40	& 	& 31.99	& 39.36	& 41.02	& 35.88	& 34.11	\\
\hline
\multirow{2}{*}{$(0.2, 0.3)$}	
	& Time Ratio	& 0.42	& 0.47	& 0.23	& 	& 5.08	& 3.79	& 1.26	& 1.14	& 1.81	\\
	& \#Calls Ratio	& 13.75	& 20.82	& 24.35	& 	& 13.37	& 19.74	& 25.20	& 19.19	& 20.06	\\
\hline
\multirow{2}{*}{$(0.2, 0.2)$}	
	& Time Ratio	& 0.57	& 0.92	& 0.26	& 	& 8.42	& 3.37	& 2.07	& 1.50	& 2.45	\\
	& \#Calls Ratio	& 21.80	& 29.62	& 25.60	& 	& 21.83	& 21.59	& 25.88	& 22.72	& 22.98	\\
\hline
\multirow{2}{*}{$(0.2, 0.1)$}	
	& Time Ratio	& 0.87	& 0.92	& 0.44	& 	& 16.69	& 3.17	& 3.61	& 2.27	& 2.60	\\
	& \#Calls Ratio	& 27.86	& 29.91	& 33.36	& 	& 34.17	& 31.58	& 40.81	& 29.01	& 29.90	\\
\bottomrule
\end{tabular}
\vspace{-1ex}
\end{table*}

Table~\ref{table:comp_param} presents more experimental results with $(\epsilon, \delta)$ parameters other than $(0.8, 0.2)$.
Nine pairs of parameters were experimented.
``Time Ratio'' represents the ratio of the running times of \texttt{ApproxMC2} to \texttt{STAC\_CNF}.
``\#Calls Ratio'' represents the ratio of the number of SAT calls of \texttt{ApproxMC2} to \texttt{STAC\_CNF}.
The results show that \texttt{ApproxMC2} gains advantage as $\epsilon$ decreases and \texttt{STAC\_CNF} gains advantage as $\delta$ decreases.
On the whole, \texttt{ApproxMC2} gains advantage when $\epsilon$ and $\delta$ both decrease.
Note that the numbers of SAT calls represent the complexity of both algorithms.
In Table~\ref{table:comp_param}, \#Calls Ratio is more stable than Time Ratio among different pairs of parameters and also different instances.
It indicates that the difficulty of NP-oracle is also an important factor of running time performance.

\subsection{Performance Comparison with Bounding and Guarantee-less Counters}

Since our approach is not a $(\epsilon, \delta)$-counter in theory,
we also compared \texttt{STAC\_CNF} with bounding counters (\texttt{SampleCount} \cite{GomesHSS07}, \texttt{MBound} \cite{GomesSS06b})
and guarantee-less counters (\texttt{ApproxCount} \cite{WeiS05}, \texttt{SampleTreeSearch} \cite{ErmonGS12}).
Table~\ref{table:comp_gl} shows the experimental results.

\begin{sidewaystable}
\centering
\caption{Performance comparison of \texttt{STAC\_CNF} with existing bounding counters and guarantee-less counters}\label{table:comp_gl}
\begin{tabular}{p{80pt}p{30pt}p{45pt} 	p{50pt}p{22pt}p{3pt}
										p{52pt}p{22pt}p{3pt}
										p{50pt}p{22pt}p{3pt}
										p{52pt}p{22pt}p{3pt}
										p{50pt}p{25pt}}
\toprule
Instance 	& $n$	& $\#F$		& \multicolumn{2}{l}{\texttt{STAC\_CNF}}	&
								& \multicolumn{2}{l}{\texttt{SampleCount}} &
								& \multicolumn{2}{l}{\texttt{MBound}}	&
								& \multicolumn{2}{l}{\texttt{ApproxCount}}	&
								& \multicolumn{2}{l}{\texttt{SampleTreeSearch}}	\\
			&		& (if known)& \multicolumn{2}{l}{($\epsilon=0.8,\delta=0.2$)}	&
								& \multicolumn{2}{l}{(99\% confidence)}	&
								& \multicolumn{2}{l}{(99\% confidence)}	&	\\
\cline{4-5}\cline{7-8}\cline{10-11}\cline{13-14}\cline{16-17}
			&		& 			& Models	& Time	&
								& L-bound	& Time	&
								& L-bound	& Time	&
								& Models	& Time	&
								& Models	& Time	\\
\midrule
blockmap\_05\_01	& 1411	& 640				& $\approx 807$					& 1 s	&
												& $\ge 22$						& 116 s	&
												& $> 64$						& 9 s	&
												& $= 640$						& 7 s	&	
												& $\approx 646$ 				& 71 s	\\
blockmap\_05\_02	& 1738	& $9.4\times 10^6$	& $\approx 7.1\times 10^6$		& 16 s	&
												& $3.6\times 10^4$				& 5 m 	&
												& $>2.1\times 10^6$				& 5 m 	&
												& $= 9.4\times 10^6$			& 13 s	&
												& $\approx 9.4\times 10^6$		& 114 s	\\
blockmap\_10\_01	& 11328	& $2.9\times 10^6$	& $\approx 2.6\times 10^6$		& 96 s	&
												& ---							& $\ge$8 h	&
												& $>5.2\times 10^5$ 			& 9 m	&
												& ---							& $\ge$8 h	&
												& $\approx 3.0\times 10^6$ 		& 62 m	\\
blockmap\_15\_01	& 33035	& ---				& $\approx 2.0\times 10^9$ 		& 41 m	&
												& ---							& $\ge$8 h	&
												& ---							& $\ge$8 h	&
												& ---							& $\ge$8 h	&
												& ---							& $\ge$8 h	\\
fs-01				& 32	& 768				& $\approx 709$					& 0.1 s &
												& $\ge 68$						& 0.2 s	&
												& $> 64$						& 2 s	&
												& $\approx 925$					& 17 s	&
												& $\approx 769$					& 0.1 s	\\
\\
\multicolumn{17}{l}{PLAN RECOGNITION}	\\
5step				& 177	& $8.1\times 10^4$	& $\approx 8.1\times 10^4$		& 0.2 s	&
												& $\ge 2.8\times 10^3$			& 4 s	&
												& $> 8.2\times 10^3$			& 6 s	&
												& $= 8.1\times 10^4$ 			& 18 s	&
												& $\approx 7.5\times 10^4$		& 1 s	\\
tire-1				& 352	& $7.3\times 10^8$	& $\approx 9.8\times 10^8$		& 64 m	&
												& $\ge 7.0\times 10^5$			& 14 s	&
												& $>6.7\times 10^7$				& 8 h	&
												& $= 7.3\times 10^8$			& 48 s	&
												& $\approx 7.6\times 10^8$ 		& 5 s	\\
tire-3				& 577	& $2.2\times 10^{11}$	& ---						& $\ge$8 h	&
												& $\ge 1.3\times 10^6$			& 49 s	&
												& ---							& $\ge$8 h	&	
												& $\approx 2.1\times 10^{11}$	& 63 s	&
												& $\approx 1.2\times 10^{11}$	& 5 s	\\
\\
\multicolumn{17}{l}{LANGFORD PROBS.}	\\
lang12				& 576	& $2.2\times 10^5$	& $\approx 3.2\times 10^5$		& 80 m	&
												& $\ge 3.6\times 10^3$			& 3 m	&
												& $>1.6\times 10^4$				& 4 h	&
												& $\approx 0$					& 111 s	&
												& $\approx 1.9\times 10^5$		& 101 m	\\
lang15				& 1024	& ---				& ---							& $\ge$8 h	&
												& $\ge 4.7\times 10^5$			& 4 m	&
												& ---							& $\ge$8 h	&	
												& $\approx 0$					& 122 s &
												& ---							& $\ge$8 h	\\
\\
\multicolumn{17}{l}{DQMR NETWORKS}	\\
or-100-20-6-UC-60	& 200	& $2.8\times 10^7$	& $\approx 3.4\times 10^7$		& 14 m	&
												& $\ge 1.1\times 10^{29}$		& 15 s	&
												& $>4.2\times 10^6$				& 6 m	&
												& $\approx 0$					& 17 s 	&
												& $\approx 2.8\times 10^7$ 		& 0.8 s	\\
or-50-10-10-UC-40	& 100	& $3.1\times 10^3$	& $\approx 3.2\times 10^3$		& 0.1 s	&
												& $\ge 2.7\times 10^2$			& 0.1 s &
												& $>5.1\times 10^2$				& 4 s	&
												& $\approx 2.1\times 10^{16}$	& 68 s &
												& $\approx 3.1\times 10^3$		& 0.5 s	\\
or-50-20-10-UC-30	& 100	& $6.8\times 10^8$	& $\approx 7.4\times 10^8$		& 35 m	&
												& $\ge 6.0\times 10^7$ 			& 0.2 s	&
												& $>1.3\times 10^8$				& 3 h 	&
												& $\approx 2.7\times 10^{16}$	& 62 s	&
												& $\approx 7.9\times 10^8$		&	0.7 s	\\
or-60-10-10-UC-30	& 120	& $6.8\times 10^7$	& $\approx 6.2\times 10^7$		& 4 m	&
												& $\ge 1.0\times 10^{17}$		& 9 s	&
												& $>1.7\times 10^7$				& 38 m	&
												& $\approx 2.4\times 10^{19}$ 	& 83 s	&
												& $\approx 6.5\times 10^7$		& 1 s	\\
or-60-5-2-UC-40		& 120	& $2.1\times 10^6$	& $\approx 1.9\times 10^6$		& 6 s	&
												& $\ge 1.0\times 10^{17}$		& 16 s	&
												& $>5.2\times 10^5$				& 199 s &
												& $\approx 2.3\times 10^{19}$ 	& 89 s	&
												& $\approx 2.1\times 10^6$		& 0.9 s	\\
or-70-10-6-UC-40	& 140	& $1.2\times 10^4$	& $\approx 7.2\times 10^3$		& 0.1 s	&
												& $\ge 1.0\times 10^{20}$		& 7 s	&
												& $>2.0\times 10^3$				& 4 s &
												& $\approx 0$ 					& 165 s	&
												& $\approx 1.2\times 10^4$		& 0.5 s	\\
or-70-5-2-UC-30		& 140	& $1.7\times 10^7$	& $\approx 4.7\times 10^7$		& 51 s	&
												& $\ge 1.0\times 10^{20}$		& 10 s	&
												& $>2.1\times 10^6$				& 11 m	&
												& $\approx 0$ 					& 165 s	&
												& $\approx 1.7\times 10^7$		& 0.8 s	\\
\\
\multicolumn{17}{l}{RANDOM 3-CNF}	\\
ran6				& 30	& $1.2\times 10^6$	& $\approx 1.9\times 10^6$		& 0.6 s	&
												& $\ge 1.1\times 10^5$			& 0.2 s	&
												& $>1.3\times 10^5$				& 11 s &
												& $\approx 8.3\times 10^5$ 		& 23 s 	&
												& $\approx 1.3\times 10^6$		& 0.2 s	\\
ran12				& 40	& $3.5\times 10^8$	& $\approx 4.2\times 10^8$		& 10 m	&
												& $\ge 3.1\times 10^7$			& 0.5 s	&
												& $>6.7\times 10^7$				& 15 m	&
												& $\approx 2.6\times 10^8$		& 23 s	&
												& $\approx 3.9\times 10^8$		& 0.3 s	\\
ran27				& 50	& $1.5\times 10^8$	& $\approx 1.2\times 10^8$		& 8 m	&
												& $\ge 1.3\times 10^7$ 			& 18 s	&
												& $>1.7\times 10^7$				& 28 m 	&
												& $\approx 5.5\times 10^7$		& 59 s 	&
												& $\approx 1.1\times 10^8$		& 0.3 s	\\
ran44				& 60	& $1.1\times 10^6$	& $\approx 1.9\times 10^6$		& 6 s	&
												& $\ge 9.5\times 10^4$			& 8 s	&
												& $>1.3\times 10^5$				& 52 s	&
												& $\approx 3.3\times 10^5$		& 139 s	&
												& $\approx 9.1\times 10^5$		& 0.5 s	\\
\bottomrule
\end{tabular}
\vspace{-1ex}
\end{sidewaystable}

For \texttt{SampleCount}, we used $\alpha = 2$ and $t = 3.5$ so that $\alpha t = 7$, giving a correctness confidence of $1-2^{-7} = 99\%$.
The number of samples per variable setting, $z$, was chosen to be 20.
Our results show that the lower-bound approximated by \texttt{SampleCount} is smaller than exact count $\#F$ by one or more orders of magnitude.
We tried larger $z$, such as $z = 100$ and $z = 1000$, but still failed to obtain a lower-bound larger than $\#F/10$.
Moreover, there are some wrong approximations on DQMR networks problems,
e.g., \texttt{or-100-20-6-UC-60} only has $2.8\times 10^7$ models but \texttt{SampleCount} returns a lower-bound $\ge 1.1\times 10^{29}$.
\texttt{SampleCount} is more efficient on Langford problems and random 3-CNF problems, but weak on problems with a large number of variables, such as \texttt{blockmap} problems.

For \texttt{MBound}, we used $\alpha = 1$ and $t = 7$ so that $\alpha t = 7$, also giving a correctness confidence of $1-2^{-7} = 99\%$.
\texttt{MBound} also employs a family of XOR hashing function which is similar to the function used by our approach.
The size of XOR constraints $k$ should be no more than half of the number of variables $n$, i.e., $k \le n/2$.
We found that XOR constraints start to fail as $k << n/2$.
So in our experiments, $k$ was chosen to be close to $n/2$.
Since \texttt{MBound} can only check the bound and may return failure as the bound too close to the exact count,
we implemented a binary search to find the best lower-bound verified by \texttt{MBound}.
The results in Table~\ref{table:comp_gl} are the best lower-bounds and the running times of the whole binary search procedure.
Though the lower-bounds are better than \texttt{SampleCount}, they are still around $\#F/10$.
Similar to our approach, the running times of \texttt{MBound} are also quite relevant to the size of $\#F$.

For \texttt{ApproxCount}, we manually increased the value of ``cutoff'' as \texttt{ApproxCount} required.
Note that \texttt{ApproxCount} calls exact model counter \texttt{Cachet} \cite{SangBBKP04} and \texttt{Relsat} \cite{BayardoS97} after formula simplifications,
so it sometimes returns the exact counts, such as \texttt{blockmap\_05\_01}, \texttt{blockmap\_05\_02}, \texttt{5step} and \texttt{tire-1}.
On Langford problems and DQMR networks problems, wrong approximations were provided.
On other instances, the results show that \texttt{STAC\_CNF} usually outperforms \texttt{ApproxCount}.

For \texttt{SampleTreeSearch}, we used its default setting about the number of samples, which is a constant.
The results show that it is very efficient and provides good approximations.
Our approach only outperforms \texttt{SampleTreeSearch} on \texttt{blockmap} problems which consist of a large number of variables.
However, there is a lack of analysis on the accuracy of the approximation of \texttt{SampleTreeSearch},
i.e., no explicit relation between the number of samples and the accuracy.

\section{Conclusion}\label{sect:conclude}
In this paper, we propose a new hashing-based approximate algorithm with dynamic stopping criterion.
Our approach has two key strengths: it requires only one satisfiability query for each cut, and it terminates once meeting the theoretical guarantee of accuracy.
We implemented prototype tools for propositional logic formulas and SMT(BV) formulas. Extensive experiments demonstrate that our approach is efficient and promising.
Despite that we are unable to prove the correctness of Equation~(\ref{eqs:ind_h}), the experimental results fit it quite well.
This phenomenon might be caused by some hidden properties of the hash functions.
To fully understand these functions and their correlation with the model count of the hashed formula might be an interesting problem to the community.
In addition, extending the idea in this paper to count solutions of other formulas is also a direction of future research.


\bibliographystyle{plain}
\bibliography{references}

\begin{thebibliography}{10}

\bibitem{BellareGP00}
M.~Bellare, O.~Goldreich, and E.~Petrank.
\newblock Uniform generation of {NP}-witnesses using an {NP}-oracle.
\newblock {\em Inf. Comput.}, 163(2):510--526, 2000.

\bibitem{BelleBP15}
V.~Belle, G.~V. Broeck, and A.~Passerini.
\newblock Hashing-based approximate probabilistic inference in hybrid domains.
\newblock In {\em Proc. of {UAI}}, pages 141--150, 2015.

\bibitem{BrownCD01}
L.~D. Brown, T.~T. Cai, and A.~Dasgupta.
\newblock Interval estimation for a binomial proportion.
\newblock {\em Statistical Science}, 16(2):101--133, 2001.

\bibitem{BrummayerB09}
R.~Brummayer and A.~Biere.
\newblock Boolector: An efficient {SMT} solver for bit-vectors and arrays.
\newblock In {\em Proc. of {TACAS}}, pages 174--177, 2009.

\bibitem{ChakrabortyFMSV14}
S.~Chakraborty, D.~J. Fremont, K.~S. Meel, S.~A. Seshia, and M.~Y. Vardi.
\newblock Distribution-aware sampling and weighted model counting for {SAT}.
\newblock In {\em Proc of {AAAI}}, pages 1722--1730, 2014.

\bibitem{ChakrabortyMMV16}
S.~Chakraborty, K.~S. Meel, R.~Mistry, and M.~Y. Vardi.
\newblock Approximate probabilistic inference via word-level counting.
\newblock In {\em Proc. of {AAAI}}, pages 3218--3224, 2016.

\bibitem{ChakrabortyMV13CAV}
S.~Chakraborty, K.~S. Meel, and M.~Y. Vardi.
\newblock A scalable and nearly uniform generator of {SAT} witnesses.
\newblock In {\em Proc. of {CAV}}, pages 608--623, 2013.

\bibitem{ChakrabortyMV13}
S.~Chakraborty, K.~S. Meel, and M.~Y. Vardi.
\newblock A scalable approximate model counter.
\newblock In {\em Proc. of {CP}}, pages 200--216, 2013.

\bibitem{ChakrabortyMV16b}
S.~Chakraborty, K.~S. Meel, and M.~Y. Vardi.
\newblock Algorithmic improvements in approximate counting for probabilistic
  inference: From linear to logarithmic {SAT} calls.
\newblock In {\em Proc. of {IJCAI}}, pages 3569--3576, 2016.

\bibitem{ChaviraD08}
M.~Chavira and A.~Darwiche.
\newblock On probabilistic inference by weighted model counting.
\newblock {\em Artif. Intell.}, 172(6-7):772--799, 2008.

\bibitem{ChistikovDM15}
D.~Chistikov, R.~Dimitrova, and R.~Majumdar.
\newblock Approximate counting in {SMT} and value estimation for probabilistic
  programs.
\newblock In {\em Proc. of {TACAS}}, pages 320--334, 2015.

\bibitem{DomshlakH07}
C.~Domshlak and J.~Hoffmann.
\newblock Probabilistic planning via heuristic forward search and weighted
  model counting.
\newblock {\em J. Artif. Intell. Res. {(JAIR)}}, 30:565--620, 2007.

\bibitem{ErmonGSS13}
S.~Ermon, C.~P. Gomes, A.~Sabharwal, and B.~Selman.
\newblock Embed and project: Discrete sampling with universal hashing.
\newblock In {\em Advances in Neural Information Processing Systems 26}, pages
  2085--2093, 2013.

\bibitem{ErmonGS12}
S.~Ermon, C.~P. Gomes, and B.~Selman.
\newblock Uniform solution sampling using a constraint solver as an oracle.
\newblock In {\em Proc. {UAI}}, pages 255--264, 2012.

\bibitem{GeldenhuysDV12}
J.~Geldenhuys, M.~B. Dwyer, and W.~Visser.
\newblock Probabilistic symbolic execution.
\newblock In {\em Proc. of {ISSTA}}, pages 166--176, 2012.

\bibitem{GomesHSS07}
C.~P. Gomes, J.~Hoffmann, A.~Sabharwal, and B.~Selman.
\newblock From sampling to model counting.
\newblock In {\em Proc. of {IJCAI}}, pages 2293--2299, 2007.

\bibitem{GomesSS06b}
C.~P. Gomes, A.~Sabharwal, and B.~Selman.
\newblock Model counting: {A} new strategy for obtaining good bounds.
\newblock In {\em Proc. of {AAAI}}, pages 54--61, 2006.

\bibitem{GomesSS06}
C.~P. Gomes, A.~Sabharwal, and B.~Selman.
\newblock Near-uniform sampling of combinatorial spaces using {XOR}
  constraints.
\newblock In {\em Advances in Neural Information Processing Systems 19}, pages
  481--488, 2006.

\bibitem{IvriiMMV16}
A.~Ivrii, S.~Malik, K.~S. Meel, and M.~Y. Vardi.
\newblock On computing minimal independent support and its applications to
  sampling and counting.
\newblock {\em Constraints}, 21(1):41--58, 2016.

\bibitem{BayardoS97}
R.~J.~Bayardo Jr. and R.~Schrag.
\newblock Using {CSP} look-back techniques to solve real-world {SAT} instances.
\newblock In {\em Proc. of {AAAI}}, pages 203--208, 1997.

\bibitem{KarpLM89}
R.~M. Karp, M.~Luby, and N.~Madras.
\newblock Monte-carlo approximation algorithms for enumeration problems.
\newblock {\em J. Algorithms}, 10(3):429--448, 1989.

\bibitem{KrocSS11}
L.~Kroc, A.~Sabharwal, and B.~Selman.
\newblock Leveraging belief propagation, backtrack search, and statistics for
  model counting.
\newblock {\em Annals of {OR}}, 184(1):209--231, 2011.

\bibitem{LiuZ11}
S.~Liu and J.~Zhang.
\newblock Program analysis: from qualitative analysis to quantitative analysis.
\newblock In {\em Proc. of {ICSE}}, pages 956--959, 2011.

\bibitem{MCVF15}
K.~S. Meel, M.~Y. Vardi, S.~Chakraborty, D.~J. Fremont, S.~A. Seshia, D.~Fried,
  A.~Ivrii, and S.~Malik.
\newblock Constrained sampling and counting: Universal hashing meets {SAT}
  solving.
\newblock In {\em Proceedings of Workshop on Beyond NP(BNP)}, 2016.

\bibitem{Roth96}
D.~Roth.
\newblock On the hardness of approximate reasoning.
\newblock {\em Artif. Intell.}, 82(1-2):273--302, 1996.

\bibitem{SangBK05}
T.~Sang, P.~Beame, and H.~A. Kautz.
\newblock Performing bayesian inference by weighted model counting.
\newblock In {\em Proc. of {AAAI}}, pages 475--482, 2005.

\bibitem{SangBBKP04}
Tian Sang, Fahiem Bacchus, Paul Beame, Henry~A. Kautz, and Toniann Pitassi.
\newblock Combining component caching and clause learning for effective model
  counting.
\newblock In {\em Proc. of {SAT}}, 2004.

\bibitem{Sipser83a}
M.~Sipser.
\newblock A complexity theoretic approach to randomness.
\newblock In {\em Proc. of the 15th Annual {ACM} Symposium on Theory of
  Computing}, pages 330--335, 1983.

\bibitem{SoosNC09}
M.~Soos, K.~Nohl, and C.~Castelluccia.
\newblock Extending {SAT} solvers to cryptographic problems.
\newblock In {\em Proc. of {SAT}}, pages 244--257, 2009.

\bibitem{Stockmeyer83}
L.~J. Stockmeyer.
\newblock The complexity of approximate counting (preliminary version).
\newblock In {\em Proc. of the 15th Annual {ACM} Symposium on Theory of
  Computing}, pages 118--126, 1983.

\bibitem{Valiant79}
L.~G. Valiant.
\newblock The complexity of enumeration and reliability problems.
\newblock {\em {SIAM} J. Comput.}, 8(3):410--421, 1979.

\bibitem{Wallis13}
S.~Wallis.
\newblock Binomial confidence intervals and contingency tests: Mathematical
  fundamentals and the evaluation of alternative methods.
\newblock {\em Journal of Quantitative Linguistics}, 20(3):178--208, 2013.

\bibitem{WeiS05}
W.~Wei and B.~Selman.
\newblock A new approach to model counting.
\newblock In {\em Proc. of {SAT}}, pages 324--339, 2005.

\bibitem{Wilson1927}
E.~B. Wilson.
\newblock Probable inference, the law of succession and statistical inference.
\newblock {\em Journal of the American Statistical Association},
  22(158):209--212, 1927.

\end{thebibliography}

\end{document}